\documentclass[11pt]{article}

\usepackage[margin=1in]{geometry}
\usepackage{microtype}
\usepackage{comment}
\usepackage{amsthm}
\usepackage{amssymb}
\usepackage{amsmath}
\usepackage{float}
\usepackage{graphicx}
\usepackage{subcaption}
\usepackage{hyperref}
\usepackage{xcolor}

\usepackage{algorithm}
\usepackage{algpseudocode}
  \usepackage{caption}
  \usepackage{tikz}

  \usetikzlibrary{arrows.meta,backgrounds,fit,calc,positioning,%
                  decorations.pathreplacing,shapes.geometric}

\newtheorem{theorem}{Theorem}
\newtheorem{lemma}{Lemma}
\newtheorem{proposition}{Proposition}
\newtheorem{corollary}{Corollary}

\theoremstyle{definition}
\newtheorem{definition}{Definition}
\newtheorem{remark}{Remark}

\newtheorem{assumption}{Assumption}

\usepackage[capitalize,noabbrev]{cleveref}
\crefname{theorem}{Theorem}{Theorems}
\crefname{proposition}{Proposition}{Propositions}
\crefname{lemma}{Lemma}{Lemmas}
\crefname{corollary}{Corollary}{Corollaries}
\crefname{definition}{Definition}{Definitions}
\crefname{remark}{Remark}{Remarks}
\crefname{assumption}{Assumption}{Assumptions}
\crefname{equation}{Equation}{Equations}
\crefname{algocf}{Algorithm}{Algorithms}   

\DeclareMathOperator{\Var}{Var}

\newcommand{\off}{\mathrm{off}}
\newcommand{\on}{\mathrm{on}}
\newcommand{\W}{\mathcal{W}}
\newcommand{\A}{\mathcal{A}}
\newcommand{\E}{\mathbb{E}}

\usepackage[round]{natbib}

\begin{document}

\begin{center}
\vspace*{-1.8em}
{\LARGE\bfseries Information-Theoretic Foundations of Online Temperature Sampling\par}
\vspace{1.15em}
\begin{tabular}{@{}c@{\hspace{1.0cm}}c@{}}
\begin{minipage}[t][0.72in][t]{0.43\textwidth}\centering
{\large\bfseries Jin Sima}\\[2pt]
{\normalsize Purdue University}\\[1pt]
{\normalsize\texttt{simaj@purdue.edu}}
\end{minipage}
&
\begin{minipage}[t][0.72in][t]{0.43\textwidth}\centering
{\large\bfseries Nikolaos Papagiannis}\\[2pt]
{\normalsize Purdue University}\\[1pt]
{\normalsize\texttt{npapagia@purdue.edu}}
\end{minipage}\\[0.95em]
\begin{minipage}[t][0.72in][t]{0.43\textwidth}\centering
{\large\bfseries Ananth Grama}\\[2pt]
{\normalsize Purdue University}\\[1pt]
{\normalsize\texttt{ayg@purdue.edu}}
\end{minipage}
&
\begin{minipage}[t][0.72in][t]{0.43\textwidth}\centering
{\large\bfseries Wojciech Szpankowski}\\[2pt]
{\normalsize Purdue University}\\[1pt]
{\normalsize\texttt{szpan@purdue.edu}}
\end{minipage}
\end{tabular}
\end{center}
\vspace{-0.65em}

\begin{abstract}
Repetitive, degenerate generation is a persistent failure of LLM decoding, and
the standard remedies (temperature, nucleus, and top-$k$ sampling) are static
heuristics. By Kac's lemma and the AEP, a stationary source takes time
exponential in its entropy to repeat a block. We therefore maximize the entropy
of the sampling distributions under one cross-entropy budget that keeps the
output plausible under the model. Offline, the solution is temperature
sampling, with the inverse temperature as the budget's Lagrange multiplier.
\emph{Our contribution is the online problem}: the conditionals
$p_1, p_2, \dots$ are revealed one at a time, and the sampling distribution
$q_t$ must be chosen causally from $p_1, \dots, p_t$ while the budget couples
the entire horizon. The optimal temperature is thus a global shadow price that
must be estimated from a prefix. We propose Entropic Online Sampling (EOS),
which re-solves for the temperature on each prefix. We prove it loses
$O(R\log n)$ entropy against the offline optimum while overshooting the budget
by $O(R\log n)$, with $R$ an instance-dependent drift parameter. A
budget-tightened variant is exactly feasible at the same rate. Conversely,
every exactly feasible online algorithm must lose $\Omega(R)$, so EOS is
optimal up to a $\log n$ factor. When the conditionals are drawn from a finite
mixture, $R=\tilde{O}(\sqrt{n})$. Experiments  confirm that EOS tracks the offline temperature and beats
fixed-temperature sampling in entropy.
\end{abstract}

\section{Introduction}
Despite the empirical success of large language models
(LLMs)~\citep{brown2020language,radford2019language}, their decoding procedures
remain poorly understood theoretically.  A central failure mode is
\emph{degenerate generation}: repetitive, low-entropy, or vacuous output.  Greedy
and beam search exacerbate it~\citep{holtzman2020curious,meister2023beam}, and the
standard remedies (e.g., temperature scaling, nucleus sampling, top-$k$ truncation,
repetition penalties) are heuristics adopted for their empirical utility rather
than derived from any objective.
Behind these heuristics lies a natural optimization problem.  Write $p_t(\cdot)
= P(\cdot \mid w^{t-1})$ for the model's conditional distribution at step $t$ and
$q_t$ for the distribution actually sampled from.  Diversity is measured by
$\sum_{t\le n} H(q_t)$ (i.e., equivalently, for stationary ergodic sources, by the recurrence-time exponent) and
plausibility by the cross-entropy of the generated stream against the model,
$-\sum_{t\le n}\sum_w q_t(w)\log p_t(w)$, which we require to stay below a budget
$\gamma n$.  Maximizing the former subject to the latter is a convex program, and
its solution is the tempered (Gibbs) family, that is, temperature sampling.

The {\it offline} problem just described is a specialization of the classical
maximum-entropy principle: a single linear constraint yields a single Lagrange
multiplier, hence a single temperature.  We record it in \Cref{sec:offline} only
because we need it as a comparator and we do not claim it as
a contribution beyond bookkeeping.  What is new in the present paper is
estimating that shadow price \emph{online}, from a prefix, without access to the horizon.  

The contribution of this paper is the \emph{online} problem, in which
$p_1,p_2,\dots$ are revealed sequentially and $q_t$ must be a function of
$p_1,\dots,p_t$ alone.  What distinguishes it from standard online convex
optimization~\citep{cesa2006prediction,orabona2019modern} is not the sequential
structure but the constraint: there is a single cross-entropy budget for the entire
horizon, coupling all $n$ rounds, rather than one constraint per round.  A causal
algorithm must therefore decide how much of the budget to spend early, without
knowing what the remaining rounds will look like.  This is a genuine
information-theoretic obstruction, and quantifying it is the substance of the paper.

One point requires care and we make it explicit throughout: \emph{which} algorithms
are being compared.  An algorithm may satisfy the global budget exactly, or may
satisfy only a running version of it at every prefix and end the horizon with
residual slack.  These are different classes, and a lower bound proved for one says
nothing about the other.  We therefore state all online results with the
feasibility class named, and we give both an exactly feasible algorithm and a
tradeoff lower bound that interpolates between the classes.

\paragraph{Our contributions.} We list our contributions next:
{\bf 1.} \textit{An explicit online protocol} (\Cref{sec:protocol}), specifying when the
policy commits, what it observes, the comparator, and the minimax quantifier order
used by the lower bound.  We define the exactly feasible class and the so called $V$-feasible
class, and identify the instance-dependent drift parameter $R$ that governs
everything.  We emphasize that $R$ is defined through the offline optimum and is
therefore \emph{not} observable to a causal algorithm.

{\bf 2.} \textit{The EOS algorithm and a bicriteria guarantee}
(\Cref{thm:upperbound}).  EOS chooses $\lambda_t$ by solving, at each prefix, the
one-dimensional monotone equation that the offline optimum would solve on the data
seen so far.  It loses at most $O(R\log n)$ entropy relative to the offline optimum
and overshoots the global budget by at most $O(R\log n)$.  We report the slack as a
first-class quantity rather than suppressing it.

{\bf 3.} \textit{An exactly feasible variant} (\Cref{thm:exactfeasible}).  Running EOS
against the tightened budget $\gamma-\varepsilon$ with $\varepsilon = \Theta(R\log
n/n)$ absorbs the terminal slack and yields \emph{exact} global feasibility, at an
additional entropy cost of $\lambda^\ast \varepsilon n = O(R\log n)$ by a
sensitivity (envelope) argument, on every instance whose slack the reserve covers.

{\bf 4.} \textit{Lower bounds and the entropy--feasibility frontier}
(\Cref{thm:lowerbound,thm:pareto}).  No exactly feasible online algorithm achieves
entropy loss $o(R)$.  More generally, for a violation budget $V\ge 0$ the attainable
loss obeys $\Omega(R) - O(V)$, recovering the previous statement at $V=0$.

{\bf 5.} \textit{When is $R$ small} (\Cref{thm:o(n)}).  If the $p_t$ are drawn
independently from a fixed set of $k$ distributions, then simultaneously over all
prefixes $R = O\!\big(\log(1/\rho)\sqrt{n(k+\log(n/\delta))}\big)$ with probability
$1-\delta$, so the entropy loss of EOS is $\tilde O(\sqrt n) = o(n)$.

\paragraph{Related work.}
That maximum-likelihood decoding degenerates is well documented.
\citet{holtzman2020curious} introduced nucleus sampling as an empirical fix;
\citet{meister2023beam} read beam search post hoc in terms of uniform
information density.  Temperature scaling, top-$k$
sampling~\citep{fan2018hierarchical}, and repetition
penalties~\citep{li2016diversity} are motivated by intuition rather than
derived; unlikelihood training~\citep{welleck2020neural} targets repetition at
training time under a static objective. Adaptive decoders such as Mirostat~\citep{basu2020mirostat} and entropy-based dynamic
temperature~\citep{zhang2024edt} adjust sampling at each step toward a per-token target
surprise or entropy; neither poses a single budget over the horizon or comes
with an optimality guarantee. Information-theoretic analyses of LLMs
have been largely descriptive: calibration and uncertainty
estimation~\citep{guo2017calibration,ovadia2019trust}, the train/generate
mismatch~\citep{meister2023beam}, entropy collapse~\citep{holtzman2020curious}.

On the algorithmic side, online convex optimization and mirror descent
\citep{cesa2006prediction,orabona2019modern} give tight regret bounds for
per-round constraints.  Constrained and long-term-constraint OCO is closer, but
there the constraints are revealed per round (or the feasible set is fixed in
advance) and regret is measured against a running average of constraint
violations.  Our single budget instead couples all $n$ rounds through one
shadow price $\lambda^\ast$ identified only in hindsight, so per-round
dual-ascent updates do not by themselves yield a feasible temperature.  Online
universal learning~\citep{wgs25} extends online techniques to
information-theoretic objectives, but without a budget coupled across the
horizon.

The closest setting is online resource allocation and online linear
programming, where a single budget is shared across the horizon and its shadow
price is learned online, by dual (mirror) descent~\citep{balseiro2020dual} or
by re-solving on the observed prefix~\citep{agrawal2014dynamic,li2022online}.
EOS is in the spirit of these re-solving schemes, specialized to a single
resource with an entropy reward: the response to a price is available in closed
form (tempering), so the online decision is one-dimensional and the price is a
temperature.  Whereas those works assume stochastic (typically i.i.d.) inputs,
we parameterize performance by the instance-dependent drift $R$, which gives
guarantees for non-stationary sequences and a matching $\Omega(R)$ lower bound
in the exactly feasible class.

The paper is organized as follows: \Cref{sec:setup} fixes notation and states the protocol.
\Cref{sec:entropy} records what entropy buys us (recurrence times) and proves the
equivalence between the probability floor and the cross-entropy constraint.
\Cref{sec:offline} is the offline warm-up.  \Cref{sec:online} contains the online
results.  \Cref{sec:R} bounds $R$.  \Cref{sec:experiments} reports experiments.

\section{Setup of Online Protocol}
\label{sec:setup}

\paragraph{Convention on logarithms.}
Throughout, $\log$ denotes the \emph{natural}
logarithm and entropies are in nats; the plausibility set is
$\W^n_\gamma=\{w^n: P(w^n)\ge e^{-\gamma n}\}$.  Every statement holds verbatim in
base $2$ after dividing $\gamma$ by $\log 2$.  (This convention removes the
$1/\ln 2$ factors from the Karush--Kuhn--Tucker conditions and makes the identity
$g_t'(\lambda)=\Var[\log 1/p_t]$ exact rather than approximate; it matters for the
constants below.)

Let $\W$ be the token alphabet, $|\W|<\infty$, and let $\W^n$ be the set of length-$n$
token sequences.  Let
\[
p_t(w) \;=\; P(w\mid w^{t-1})\in\Delta(\W)
\]
be the conditional distribution produced by the model given the prefix
$w^{t-1}=(w_1,\dots,w_{t-1})$, so that
$P(w^n)=\prod_{t=1}^{n}P(w_t\mid w^{t-1})=\prod_{t=1}^n p_t(w_t)$.  This is the exact
chain rule, not an independence assumption.  We write
$\W^n_\gamma=\{w^n: P(w^n)\ge e^{-\gamma n}\}$ for the set of \emph{plausible}
sequences.
Tokens are drawn as $w_t\sim q_t$ where $q_t=\phi_t(p_1,\dots,p_t)\in\Delta(\W)$, so
the sampled sequence has law $\prod_t q_t(w_t)$.  Top-$k$, top-$p$, and temperature
sampling are all instances of this.

\paragraph{Conditionals.}
In our analysis $p_1,\dots,p_n$  are chosen by the
environment of the protocol and are unaffected by the sampled tokens.  Under
autoregressive feedback, the sequence-level optimum
$Q^\ast(w^n)\propto P(w^n)^{\lambda}$ has conditionals
$Q^\ast(w_t\mid w^{t-1})\propto p_t(w_t)^{\lambda}Z_{t+1}(w^t)$ with
$Z_{t+1}(w^t)=\sum_{w_{t+1}^n}\prod_{s>t}p_s(w_s)^{\lambda}$, which reduces to
temperature sampling exactly when $Z_{t+1}$ does not depend on $w_t$.  Our guarantees are stated for this regime; \Cref{alg:eos}
runs unchanged under feedback and is evaluated there in
\Cref{sec:experiments}.

\begin{assumption}[Bounded $p_t$]
\label{as:rho}
There is $\rho>0$ with $p_t(w)\ge\rho$ for all $t\in[n]$, $w\in\W$.
\end{assumption}
This is the standard consequence of top-$k$ / nucleus pre-filtering followed by
renormalization (EOS then acts on this support), and it is what makes the cross-entropy per round bounded:
$0\le -\sum_w q(w)\log p_t(w)\le \log(1/\rho)$ for every $q\in\Delta(\W)$.
Furthermore,
for $\lambda\ge 0$ define the tempered distribution and the two scalar functions we
use throughout:
\begin{align}
q^{\lambda}_t(w)&=\frac{p_t^{\lambda}(w)}{\sum_{w'\in\W}p_t^{\lambda}(w')},
\label{eq:tempered}\\
f_t(\lambda)&=H\!\left(q^\lambda_t\right),
\qquad
g_t(\lambda)=\gamma+\sum_{w\in\W}q^{\lambda}_t(w)\log p_t(w) .
\label{eq:ftgt}
\end{align}
Thus $-\sum_w q^\lambda_t(w)\log p_t(w)=\gamma-g_t(\lambda)$ is the per-round
cross-entropy, and $g_t(\lambda)\ge 0$ means round $t$ is within budget.

\paragraph{The offline optimization problems.}
The \emph{offline} problem, in which all of $p_1,\dots,p_n$ are known in advance, is
\begin{align}
&\max_{q_t\in\Delta(\W),\,t\in[n]}\ \sum_{t=1}^n H(q_t)\\
\qquad\text{s.t.}\qquad&
-\sum_{t=1}^n\sum_{w\in\W}q_t(w)\log p_t(w)\;\le\;\gamma n .
\label{eq:offline}
\end{align}
Its value is denoted $\mathrm{OPT}(\gamma n)$ and its solution $Q^\off_t$.

\paragraph{The online protocol.} \label{sec:protocol}
We state the protocol explicitly, since every subsequent claim depends on the
quantifier order.  The horizon $n$, the budget $\gamma$, the alphabet $\W$ and the floor
$\rho$ are public.  The algorithm $\A=(\phi_1,\dots,\phi_n)$ is \emph{committed ex
ante}, before any $p_t$ is revealed.

\begin{quote}
For $t=1,\dots,n$: the environment reveals $p_t\in\Delta(\W)$; the
algorithm outputs $q^{\A}_t=\phi_t(p_1,\dots,p_t)$; the token $w_t\sim q^\A_t$ is
emitted.
 The environment is adversarial in the sense that the sequence
$p_1,\dots,p_n$ is chosen by an adversary that knows $\A$ but must produce $p_t$
before seeing $q^\A_t$; equivalently, lower bounds are stated as
$\exists\{p_t\}$ after $\forall\A$.
\end{quote}

The online optimization problem then can be formulated as follows:
\begin{equation}\label{eq:diversity}
\begin{aligned}
&\max_{q_t=\phi_t(\{p_i\}^t_{i=1})\in\Delta(\mathcal{W}):\,t\in[n]}\ \sum^n_{t=1}H(q_t)\\
&\text{s.t.}\quad -\sum^n_{t=1}\sum_{w\in\mathcal{W}}q_t(w)\log p_t(w)\le \gamma n .
\end{aligned}
\end{equation}
Performance is measured by the \emph{entropy loss}
$$\sum_{t}H(Q^{\off}_t)-\sum_t H(q^\A_t)$$
against the offline optimum of
\eqref{eq:offline} on the same instance, together with the \emph{global slack}
$\big[-\sum_t\sum_w q^\A_t(w)\log p_t(w)-\gamma n\big]_+$.
The offline optimum is the honest comparator: it upper-bounds the value of every
causal policy, so a guarantee stated against it is stronger than the same guarantee
stated against a causal reference.

\begin{definition}[Feasibility classes]
\label{def:classes}
Let $V\ge 0$.  An online algorithm $\A$ is \emph{$V$-feasible} if for every instance
\[
-\sum_{t=1}^n\sum_{w\in\W}q^{\A}_t(w)\log p_t(w)\;\le\;\gamma n + V .
\]
$0$-feasible algorithms are called \emph{exactly feasible}.  We write $\A_V$ for the
class of $V$-feasible algorithms; $\A_0\subseteq \A_V$ for all $V$.
\end{definition}

\begin{definition}[Drift parameter]
\label{def:R}
With $Q^\off$ the offline optimum of \eqref{eq:offline}, set
\begin{equation}
\label{eq:R}
R_t \;=\; \Big|\sum_{i=1}^t\sum_{w\in\W}Q^{\off}_i(w)\log p_i(w)+\gamma t\Big|
\;=\;\Big|\sum_{i=1}^t g_i(\lambda^\ast)\Big|
\end{equation}
and $R \;=\; R(n)\;=\;\max_{t\in[n]}R_t$, where $\lambda^\ast$ is defined in (\ref{eq:lambdastar}).
\end{definition}
We note that $R_t$ measures how far the offline-optimal policy is, at prefix $t$, from having
spent exactly a $t/n$ fraction of its budget.  If the $p_t$ are exchangeable in the
strong sense that every prefix looks like the whole, $R$ is small; if the instance
front-loads or back-loads its hard rounds, $R$ is large.

\begin{remark}[$R$ is not observable]
\label{rem:Rnotobservable}
$R$ is defined through $Q^\off$, hence through $\lambda$, hence through the
entire horizon.  It is an \emph{instance-complexity} parameter, not a quantity
available to a causal algorithm.  All our upper bounds are stated as
``$O(R\log n)$'' with $R$ appearing only in the analysis; the one place where the
algorithm would need to know $R$ is the tightening parameter of
\Cref{thm:exactfeasible}, and we address that explicitly in
\Cref{rem:tuning}.
\end{remark}

In passing we observe that
both the offline program and the online algorithm are defined through the root of a
monotone scalar equation, and that root need not exist for every $\gamma$.  We
record the condition.

\begin{assumption}[Strict feasibility on every prefix]
\label{as:feasible}
For every $t\in[n]$, the budget $\gamma$ satisfies
\begin{equation}\label{eq-gamma}
\begin{aligned}
&\frac1t\sum_{i=1}^{t}\log\frac{1}{\max_{w}p_i(w)}\;<\;\gamma\\
&\qquad\le\;\frac1t\sum_{i=1}^{t}\frac{1}{|\W|}\sum_{w\in\W}\log\frac{1}{p_i(w)} .
\end{aligned}
\end{equation}
\end{assumption}

The left inequality says the budget is achievable at all (even the greedy policy
$\lambda\to\infty$ has cross-entropy $\frac1t\sum_i\log(1/\max_w p_i)$, so a smaller
$\gamma$ is infeasible); the right says it is binding (at $\lambda=0$, i.e.\ uniform
sampling, the cross-entropy is already at or above budget, so $\lambda^\ast\ge 0$).
If the right inequality fails the constraint is inactive and uniform sampling is
optimal --- a degenerate case we exclude.  Under
\Cref{as:rho,as:feasible} the map $\lambda\mapsto\sum_{i\le t}g_i(\lambda)$ is
continuous and nondecreasing, negative at $\lambda=0$ and positive as
$\lambda\to\infty$, so a root exists; \Cref{as:curvature} below makes it unique.

\begin{assumption}[Curvature]
\label{as:curvature}
There is $c>0$ such that for all $t\in[n]$,
\begin{equation}
\label{eq-var}
\sum_{i=1}^{t}\Var_{Q^{\off}_i}\!\left[\log\tfrac{1}{p_i(w)}\right]\ \ge\ ct .
\end{equation}
\end{assumption}
Since $g_t'(\lambda)=\Var_{q^\lambda_t}[\log(1/p_t)]$ (\Cref{prop:derivative}), this
says the prefix constraint functions have positive slope at $\lambda^\ast$: the temperature
is identified at rate $1/t$ from the prefix.  It fails only if the conditionals are
asymptotically deterministic or asymptotically uniform.  \Cref{fig::eli5_variance}
gives empirical support.

\subsection{Entropy controls recurrence time}
\label{sec:entropy}

We first make precise the sense in which maximizing entropy provably delays
repetition and prompts diversity. This is a classical theorem, and we state it as such rather than as an
analogy but we also state its scope, because it applies literally only in the
stationary specialization.

\begin{proposition}[Recurrence-time theorem; ~\cite{wynerziv1989}, ~\cite{ornsteinweiss1993}]
\label{prop:recurrence}
Let $\{W_i\}$ be a stationary ergodic source with entropy rate $H$, and let
$\mathcal{R}_n$ denote the recurrence time of the length-$n$ prefix, i.e.\ the first
$m>0$ with $W_{m+1}^{m+n}=W_1^{n}$.  Then
\[
\frac{1}{n}\log \mathcal{R}_n \;\longrightarrow\; H \qquad\text{almost surely.}
\]
Consequently the time before the generated stream repeats a length-$n$
block grows as $e^{nH(1+o(1))}$ almost surely: higher entropy provably yields exponentially longer
recurrence times.
\end{proposition}

The intuition is Kac's lemma~\citep{kac1947} --- for a stationary ergodic source the
mean return time to a state is the reciprocal of its probability --- combined with
the AEP~\citep{thomas2006elements}, by which a typical $n$-block has probability
$e^{-nH}$. However, see \citep{spa93} for a rephrasing of this problem from the suffix-tree point of view.

\begin{remark}[Scope]
\label{rem:scope}
\Cref{prop:recurrence} applies verbatim to fixed-temperature sampling in the regime of
\Cref{thm:o(n)}, where the generated stream is i.i.d.\ and hence stationary ergodic;
there, maximizing $\sum_t H(q_t)$ maximizes the recurrence exponent.  In the general
non-stationary setting the AEP half of the argument survives and is exactly what
\Cref{thm:probabilityexplanation} establishes non-asymptotically, while the Kac
``return time $=1/$probability'' step does not.  We therefore adopt
$\sum_t H(q_t)$ as the objective on the strength of the stationary case and analyze
it on its own terms; we do not claim a recurrence-time theorem in the general case.
\end{remark}

The constraint we actually impose is the cross-entropy budget.  The next result says
it is equivalent, up to $O(\sqrt{n\log(1/\delta)})$, to the probability floor
$w^n\in\W^n_\gamma$  which is the constraint one would write down from first
principles.  In the language of information geometry, the budget defines a
cross-entropy (KL) ball around the model and the tempered family
\eqref{eq:tempered} is the corresponding $I$-projection geodesic; the ``floor'' is
not an ad hoc device.

\begin{theorem}
\label{thm:probabilityexplanation}
Let $p_1,\dots,p_n$ be given, $w^n$ sampled sequentially from $q_1,\dots,q_n$, and let \Cref{as:rho} hold.
Write $L=\log(1/\rho)$.
\begin{enumerate}
\item[(i)] If
\[
\sum_{t\in[n]}\sum_{w}q_t(w)\log p_t(w)+\gamma n\ \ge\ -\,L\sqrt{\tfrac{n}{2}\log\tfrac{2}{\delta}},
\]
then
\[
\Pr\!\left(P(w^n)\ \ge\ e^{-\gamma n-2L\sqrt{(n/2)\log(2/\delta)}}\right)\ \ge\ 1-\delta .
\]
\item[(ii)] Conversely, if $\Pr(P(w^n)\ge e^{-\gamma n - \Delta_n})\ge 1-\delta$ then
\[
\begin{aligned}
&\sum_{t\in[n]}\sum_{w}q_t(w)\log p_t(w)+\gamma n\\
&\qquad\ge\; -\Delta_n - L\sqrt{\tfrac{n}{2}\log\tfrac{2}{\delta}} - \delta n L .
\end{aligned}
\]
\end{enumerate}
Choosing $\delta=O(n^{-1/2})$, the constraint $w^n\in\W^n_\gamma$ and the
cross-entropy constraint \eqref{eq:offline} agree to within
$O(\sqrt{n\log n})$, i.e.\ they are equivalent at the level of the exponent.
\end{theorem}

\section{Main Results}

In this section we formulate our main results starting with a warm-up exercise regarding the offline optimization which we use as a comparator.
Then we solve the online optimization problem and propose Algorithm~\ref{alg:eos} that we prove achieves asymptotically the performance of the online optimization.

\subsection{Offline Warm-Up: the Comparator}
\label{sec:offline}

To start, we discuss briefly the offline optimization problem since we use it as a comparator for the online optimization.

\begin{proposition}[Offline optimum is fixed-temperature sampling]
\label{thm:offline}
Under \Cref{as:rho,as:feasible}, the solution of \eqref{eq:offline} is
$Q^{\off}_t=q^{\lambda^\ast}_t$ as in \eqref{eq:tempered}, where $\lambda^\ast\ge 0$
solves
\begin{equation}\label{eq:lambdastar}
\begin{gathered}
-\sum_{t\in[n]}\sum_{w\in\W}q^{\lambda^\ast}_t(w)\log p_t(w)\;=\;\gamma n,\\
\text{equivalently}\quad \sum_{t\in[n]}g_t(\lambda^\ast)=0 .
\end{gathered}
\end{equation}
Under \Cref{as:curvature} the root is unique.  The optimal inverse temperature is
common to all $t$: $\lambda^\ast=1/T^\ast$.
\end{proposition}

The proof is the Karush--Kuhn--Tucker system for a concave objective under one
linear constraint, and is given in \Cref{app:offline} for completeness.  We
underline the reading: \emph{one} global constraint yields \emph{one} multiplier,
hence \emph{one} temperature.  This is why temperature sampling --- as opposed to any
other reweighting of $p_t$ --- is the correct functional form, and it is also why the
online problem is nontrivial: a causal algorithm must estimate a global shadow price
from a prefix.

\begin{lemma}[Sensitivity / envelope]
\label{lem:envelope}
Let $\mathrm{OPT}(b)$ denote the optimal value of \eqref{eq:offline} with the budget
$\gamma n$ replaced by $b$.  Then $\mathrm{OPT}$ is concave and nondecreasing on its
domain, and $\mathrm{OPT}'(b)=\lambda^{\ast}(b)$, the optimal multiplier at budget
$b$.  In particular, for $b'\le b$,
\[
0\ \le\ \mathrm{OPT}(b)-\mathrm{OPT}(b')\ \le\ \lambda^\ast(b')\,(b-b') .
\]
\end{lemma}

Lemma~\ref{lem:envelope} is the workhorse of \Cref{sec:online}: it converts every
statement about \emph{budget} into a statement about \emph{entropy} at the exchange
rate $\lambda^\ast$.  

\subsection{The Online Problem}
\label{sec:online}

Next, we provide a solution to the {\it online optimization problem}.

\begin{theorem}[Clairvoyant benchmark]\label{lem:onlinelambda}
For a fixed instance $p_1,\dots,p_n$, the online program \eqref{eq:diversity} attains the
same optimal value and optimizer as the offline program \eqref{eq:offline}: the optimum is
\begin{align}\label{eq:qon}
q_t(w)=Q^{\on}_t=\frac{p^{\lambda^{\on}}_t(w)}{\sum_{w\in\mathcal{W}}p^{\lambda^{\on}}_t(w)}, ~ ~ ~ t\in[n],
\end{align}
with a \emph{single} exponent $\lambda^{\on}=\lambda^\ast$ common to all $t$, fixed by the
global budget
\begin{align}\label{eq:onlineprobconstraint}
   - \sum_{t\in[n]}\sum_{w\in\mathcal{W}}\frac{p^{\lambda^{\on}}_t(w)}{\sum_{w\in\mathcal{W}}p^{\lambda^{\on}}_t(w)}\log p_t(w)=\gamma n .
\end{align}
This is a clairvoyant benchmark, not an admissible policy under the protocol (which
commits $\A$ before the instance); the difficulty of the online problem is purely \emph{informational},
since $\lambda^\ast$ is a function of the entire horizon and cannot be computed from a prefix.
\end{theorem}
\begin{proof}
For a fixed instance $(p_1,\dots,p_n)$ each $q_t$ is trivially a (constant) function of
$p_1,\dots,p_t$, so the causal constraint in \eqref{eq:diversity} does not shrink the
feasible set relative to \eqref{eq:offline}; the two programs coincide. \Cref{thm:offline}
then yields the optimizer $q^{\lambda^\ast}_t$ with a \emph{single} multiplier $\lambda^\ast$
(one budget, one Lagrange multiplier, one temperature), and complementary slackness under
\Cref{as:feasible} forces the budget to bind, giving \eqref{eq:onlineprobconstraint}. A
per-round exponent $\lambda^{\on}_t$ would leave the allocation of cross-entropy across
rounds undetermined; since $f_t'(\lambda)=-\lambda\,g_t'(\lambda)$ (\Cref{prop:derivative}),
equalizing the marginal entropy-per-nat-of-budget across $t$ is exactly the
common-$\lambda^\ast$ solution, so no per-round choice can improve $\sum_t H(q_t)$.
\end{proof}

\subsection{Entropic Online Sampling}

Unfortunately, computing $Q_t^{\on}$ requires the horizon-dependent optimal temperature $\lambda^\ast$ and is therefore not available to a causal algorithm. Therefore, we propose a simpler and efficient algorithm, called Entropic Online Sampling (EOS) Algorithm~\ref{alg:eos}. It is the natural causal surrogate for \Cref{thm:offline}: at each prefix, solve
the offline problem restricted to the data seen so far. The rest of this section is devoted to proving EOS is asymptotically optimal.

Equation \eqref{eq:lambdat} of Algorithm~\ref{alg:eos} is one-dimensional with a nonincreasing left-hand
side, so bisection or Newton's method solves it to precision $\epsilon$ in
$O(\log(1/\epsilon))$ evaluations.  Each evaluation costs $O(t|\W|)$, since
every past round must be re-tempered at the new $\lambda$.  The clipping step is what makes the algorithm well defined
when \Cref{as:feasible} fails on a short prefix --- a case that does occur early in
generation, when a few near-deterministic rounds can put the prefix budget out of
reach.  Note the contrast with the offline solution: EOS uses a single exponent
$\lambda_t$ across all rounds $i\le t$ of the prefix, updated as $t$ grows, whereas
$\lambda^\ast$ is a single exponent fixed with hindsight.

\begin{algorithm}[H]
\caption{Entropic Online Sampling (EOS)}\label{alg:eos}
\begin{algorithmic}[1]
\Require budget $\gamma$; cap $\Lambda$
\For{$t=1,\dots,n$}
  \State observe $p_t$
  \State solve for $\lambda_t\in[0,\Lambda]$ the monotone scalar equation
  \begin{equation}\label{eq:lambdat}
    -\sum_{i\in[t]}\sum_{w\in\W}q^{\lambda_t}_i(w)\log p_i(w)=\gamma t,
  \end{equation}
  \Statex \hspace{\algorithmicindent} setting $\lambda_t=0$ if $\sum_{i\in[t]}g_i(\lambda)>0$ for all $\lambda\in[0,\Lambda]$, and $\lambda_t=\Lambda$ if it is $<0$ throughout
  \State emit $w_t\sim q^{\lambda_t}_t$
\EndFor
\end{algorithmic}
\end{algorithm}

We need the following estimate of $\lambda_t$ to prove asymptotic optimality of Algorithm~\ref{alg:eos}.

\begin{lemma}[Bounded multiplier]
\label{lem:lambdabound}
Under \Cref{as:rho,as:feasible}, $\lambda_t,\lambda^\ast\ge0$; if \eqref{eq-var}
holds with $q^\lambda_i$ for all $\lambda\in[0,\Lambda]$,
$\Lambda:=(\log(1/\rho)-\gamma)/c$, then $\lambda_t,\lambda^\ast\le\Lambda$.
Otherwise $\Lambda$ is any upper bound on $\lambda^\ast$ (\Cref{alg:eos} clips $\lambda_t$ to $[0,\Lambda]$).
\end{lemma}

\subsection{Asymptotic Optimality of EOS}

We start with a general lower bound proved in Appendix~\ref{app:lowerbound}.

\begin{theorem}[Lower bound, $\A_0$]
\label{thm:lowerbound}
Let $\gamma$ and $|\W|$ satisfy
\[
-\frac{(|\W|-1)\log\!\big(\tfrac{1-e^{-\gamma/3}}{|\W|-1}\big)}{|\W|}+\frac{\gamma}{3|\W|}>\frac{4\gamma}{3}
\]
and $\log|\W|\ \ge\ \frac{4\gamma}{3}$,
and let $R\le \gamma n/6$.  Then for every exactly feasible online algorithm
$\A\in\A_0$ there exists an instance $\{p_t\}_{t\in[n]}$ with drift parameter
$R(n)=R$ such that
\begin{equation}
\label{eq:divlower}
\sum_{t=1}^n H(Q^{\off}_t)-\sum_{t=1}^n H(q^{\A}_t)\ \ge\ \kappa\,R,
\qquad
\kappa \;=\; \frac{\gamma^{3}}{216\,G^{2}},
\end{equation}
and $G:=\Big(\tfrac{4\gamma}{3}+\log\tfrac{1}{1-e^{-\gamma/3}}\Big)^{2}$.
\end{theorem}

The construction is a pair of instances
agreeing on a common prefix of length $3R/\gamma$ and diverging afterwards; because
$\A$ commits before seeing the divergence, it must misallocate budget on one of the
two, and \Cref{lem:envelope} converts the misallocation into entropy.

\paragraph{The entropy--feasibility frontier.}

\Cref{thm:lowerbound} constrains $\A_0$.  It says nothing about an algorithm that
enforces only a running constraint and finishes with residual slack, since such an
algorithm is not in $\A_0$.  The right object is a frontier.

\begin{theorem}[Pareto lower bound]
\label{thm:pareto}
Fix $V\ge 0$.  Under the hypotheses of \Cref{thm:lowerbound}, for every
$V$-feasible online algorithm $\A\in\A_V$ there exists an instance with drift
parameter $R$ such that
\begin{equation}
\label{eq:pareto}
\sum_{t=1}^n H(Q^{\off}_t)-\sum_{t=1}^n H(q^{\A}_t)\ \ge\ \kappa R\;-\;\Lambda V ,
\end{equation}
with $\kappa$ as in \eqref{eq:divlower} and $\Lambda\ge$ the tail multipliers
$\lambda^\ast_x$ (\Cref{app:pareto}).
\end{theorem}

At $V=0$ this is \Cref{thm:lowerbound}.  For $V>0$ it quantifies the exchange rate:
enlarging the feasible set by $V$ nats of budget buys at most $\Lambda V$ nats of
entropy, and the bound is nonvacuous as long as $V<\kappa R/\Lambda$.  The proof
(\Cref{app:pareto}) is a one-line modification of \Cref{app:lowerbound}: slack $V$
loosens the tail subproblem from $-R-B$ to $-R-B-V$, and \Cref{lem:envelope} bounds
the resulting gain by $\Lambda V$.

\begin{theorem}[Bicriteria guarantee for EOS]
\label{thm:upperbound}
Let $q^{\mathrm{on}}_t$ be the output of \Cref{alg:eos} and let
\Cref{as:rho,as:feasible,as:curvature} hold.  Then there is a constant
$C=C(\rho,c,\Lambda,|\W|)$ such that
\begin{align}
&\sum_{t=1}^{n} H(Q_t^{\off}) - \sum_{t=1}^{n} H(q_t^{\mathrm{on}})&\ \le\ C\,R \log n,
\label{eq:divupper}\\
&-\sum_{t=1}^{n} \sum_{w \in \W} q_t^{\mathrm{on}}(w)\log p_t(w) - \gamma n&\ \le\ C\,R\log n .
\label{eq:violationupper}
\end{align}
\end{theorem}

We report \eqref{eq:violationupper} as a first-class quantity: EOS is
$O(R\log n)$-feasible, not exactly feasible, and the pair
\eqref{eq:divupper}--\eqref{eq:violationupper} is a point on the frontier of
\Cref{thm:pareto} rather than a comparison against \Cref{thm:lowerbound}.  The next
subsection removes the slack.

We now establish a feasible upper bound showing that our Algorithm~\ref{alg:eos} achieves asymptotic optimality.

\begin{theorem}[EOS with budget tightening]
\label{thm:exactfeasible}
Fix $\varepsilon\in(0,\gamma)$ and run \Cref{alg:eos} with $\gamma$ replaced by
$\gamma_\varepsilon:=\gamma-\varepsilon$ in \eqref{eq:lambdat}.  Let
$R_\varepsilon:=R(n;\gamma_\varepsilon)$ be the drift parameter at the tightened
budget and let \Cref{as:rho,as:feasible,as:curvature} hold at
$\gamma_\varepsilon$.  With $C$ as in \Cref{thm:upperbound} and
cap $\Lambda_\varepsilon\ge\lambda^\ast(\gamma_\varepsilon)$ ($=\Lambda+\varepsilon/c$ under \Cref{lem:lambdabound}):
\begin{enumerate}
\item[(i)] \emph{(Exact feasibility.)}  If $\varepsilon n \ge C R_\varepsilon\log n$,
the output satisfies the global constraint of \eqref{eq:offline} exactly:
$-\sum_t\sum_w q_t(w)\log p_t(w)\le\gamma n$, i.e.\ it is exactly feasible on every such instance.
\item[(ii)] \emph{(Entropy.)}  Against the offline optimum at the \emph{original}
budget,
\[
\sum_{t=1}^n H(Q^{\off}_t)-\sum_{t=1}^n H(q_t)\ \le\ C R_\varepsilon\log n \;+\; \Lambda_\varepsilon\,\varepsilon n .
\]
\end{enumerate}
In particular, $\varepsilon$ with $\varepsilon n = C R_\varepsilon \log n$ (a fixed
point, since $R_\varepsilon$ depends on $\varepsilon$), when it exists, gives an exactly
feasible algorithm with entropy loss at most $(1+\Lambda_\varepsilon)\,C R_\varepsilon\log n =
O(R_\varepsilon\log n)$.
\end{theorem}

The proof (\Cref{app:exactfeasible}) is short: (i) is \eqref{eq:violationupper}
applied at budget $\gamma_\varepsilon$, whose slack of at most $CR_\varepsilon\log n$
is absorbed by the $\varepsilon n$ nats held in reserve; (ii) combines
\eqref{eq:divupper} at $\gamma_\varepsilon$ with \Cref{lem:envelope}, which prices
the reserve at $\lambda^\ast\le\Lambda_\varepsilon$ nats of entropy per nat of budget.

\begin{remark}[Why the terminal multiplier cannot do this]
\label{rem:terminal}
A tempting alternative is to leave EOS unchanged and fix the final multiplier
$\lambda_n$ against the realized prefix so that the global constraint holds with
equality.  This does not work.  The accumulated slack after $n-1$ rounds is
$\Theta(R\log n)$, which for $R=\Theta(\sqrt n)$ is $\Theta(\sqrt n\log n)$, whereas
a single round can move the cross-entropy sum by at most $\log(1/\rho)=O(1)$.  The
correction must be spread over $\Theta(R\log n)$ rounds --- or, as here, built into
the budget from the start.
\end{remark}

\begin{remark}[Tuning $\varepsilon$ without knowing $R$]
\label{rem:tuning}
By \Cref{rem:Rnotobservable}, $R_\varepsilon$ is not observable online.  Three
options.  (a) In the regime of \Cref{sec:R}, $R_\varepsilon =
O(\sqrt{n(k+\log(n/\delta))})$ with probability $1-\delta$, so
$\varepsilon=\Theta(\sqrt{(k+\log (n/\delta))/n}\,\log n)$ is a valid a priori choice.
(b) Use the observable proxy
$\hat R_t=\max_{i\le t}|\sum_{j\le i}g_j(\lambda_t)|$, which lower-bounds $R$ up to
the estimation error $|\lambda_t-\lambda^\ast|=O(R/(ct))$, and tighten adaptively;
this costs a further constant factor.  (c) A doubling schedule on $\varepsilon$,
restarting when the running slack exceeds the reserve.
\end{remark}

\begin{corollary}[Optimality up to a logarithm]
\label{cor:tight}
Given a bound on $R_\varepsilon$, within the exactly feasible class $\A_0$, the attainable entropy loss is at least
$\kappa R$ (\Cref{thm:lowerbound}) and at most $O(R_\varepsilon\log n)$
(\Cref{thm:exactfeasible}).  The two agree up to a single factor of $\log n$
(and up to the difference between $R$ and $R_\varepsilon$, see
\Cref{rem:Reps}).
\end{corollary}

\begin{remark}[$R$ versus $R_\varepsilon$]
\label{rem:Reps}
Tightening the budget perturbs the comparator, so \Cref{thm:exactfeasible} is stated
in terms of $R_\varepsilon$.  To first order the perturbation cancels: writing
$A_t(\lambda)=\sum_{i\le t}g_i(\lambda)$, the tightened root satisfies
$\lambda^\ast(\gamma_\varepsilon)-\lambda^\ast(\gamma)\approx \varepsilon/c$, and
\[
A_t^{\varepsilon}\big(\lambda^\ast(\gamma_\varepsilon)\big)
\;\approx\; A_t\big(\lambda^\ast(\gamma)\big) + \frac{\varepsilon}{c}\sum_{i\le t}g_i'-\varepsilon t
\;\approx\;A_t\big(\lambda^\ast(\gamma)\big),
\]
since $\sum_{i\le t}g_i'\approx ct$.  The residual is controlled by the
inhomogeneity of the curvature across the horizon, and vanishes when the $g_i'$ are
exchangeable.  In the regime of \Cref{sec:R} the two are of the same order, since the
bound of \Cref{thm:o(n)} holds uniformly in $\gamma$ over compact subsets of the
feasible range; so nothing is lost in the headline corollary
$\tilde O(\sqrt n)$.  We do not claim $R_\varepsilon=\Theta(R)$ in general.
\end{remark}

\paragraph{When Is the Drift Parameter $R$ Small?}
\label{sec:R}

In the worst case $R=\Theta(n)$ and the guarantees are vacuous, as they must be:
\Cref{thm:lowerbound} exhibits such instances.  The question is whether $R=o(n)$
under conditions resembling real decoding.  We analyze the simplest stationarity
proxy: the $p_t$ are drawn independently from a finite set of $k$ distributions.  We
present this as a first tractable case within our regime of
\Cref{sec:setup}; it does not capture the dependence of $p_t$
on previously sampled tokens.

\begin{lemma}[$\ell_1$ concentration; Bretagnolle--Huber-Carol inequality, see \citet{vandervaart1996weak}]
\label{lem:empiricaltv}
Let $p_t$ be drawn independently with $\Pr(p_t=p^{(i)})=P_i$, $i\in[k]$, and let
$\hat P_i=|\{t\le n: p_t=p^{(i)}\}|/n$.  Then with probability at least $1-\delta$,
\[
\sum_{i\in[k]}\big|\hat P_i-P_i\big|\ \le\ \sqrt{\frac{2\big(k\log 2+\log\tfrac{1}{\delta}\big)}{n}} .
\]
\end{lemma}
\begin{proof}
Write $\sum_{i}|\hat P_i-P_i|=2\sup_{A\subseteq[k]}\big(\sum_{i\in A}\hat P_i-\sum_{i\in A}P_i\big)$.
For fixed $A$, one-sided Hoeffding gives
$\Pr(\sum_{i\in A}(\hat P_i-P_i)\ge\epsilon)\le e^{-2n\epsilon^2}$; a union bound
over the $2^k$ subsets and solving $2^{k}e^{-2n\epsilon^2}=\delta$ bounds the supremum by
$\sqrt{(k\log 2+\log\frac1\delta)/(2n)}$, and doubling yields the claim.
\end{proof}

\begin{theorem}[$R$ is $\tilde O(\sqrt n)$, uniformly over prefixes]
\label{thm:o(n)}
In the setting of \Cref{lem:empiricaltv}, with probability at least $1-\delta$,
\emph{simultaneously for all $t\in[n]$},
\begin{align*}
R &\;=\; \max_{t\in[n]} R_t
\;\le\; 2\log\tfrac{1}{\rho}\,\sqrt{2n\Big(k\log 2+\log\tfrac{n+1}{\delta}\Big)}\\
&\;=\; O\!\left(\log\tfrac{1}{\rho}\sqrt{n\big(k+\log\tfrac{n}{\delta}\big)}\right).
\end{align*}
Consequently, by \Cref{thm:upperbound,thm:exactfeasible}, the entropy loss of EOS
and of its exactly feasible variant is $\tilde O(\sqrt n)=o(n)$.
\end{theorem}

\section{Experiments}
\label{sec:experiments}

We validate three predictions: (i) EOS stays close to the offline optimum and
above budget-feasible temperatures fixed in advance; (ii) the
curvature condition \Cref{as:curvature}; (iii) that EOS tracks
$\lambda^\ast$.  The comparison in (i) is motivated by \Cref{thm:offline}: a
fixed temperature is optimal only if it equals $1/\lambda^\ast$, which depends
on the whole horizon, so a temperature chosen in advance either exceeds the
budget or leaves budget unspent, which by \Cref{lem:envelope} costs entropy at
rate $\lambda^\ast$.

We use ELI5~\citep{fan2019eli5} and Alpaca~\citep{alpaca} with Meta-Llama-3-8B-Instruct~\citep{grattafiori2024llama},
$\gamma=0.9$ and a probability cutoff $\rho=10^{-4}$ (realizing
\Cref{as:rho}).  The fixed baseline is $T=0.7$, a common default that spends
well below the budget (e.g.\ $0.35$ nats per token in the autoregressive run
below), so it is feasible but leaves budget unspent.  On both datasets EOS
tracks the offline optimum with a small gap and exceeds $T=0.7$ in entropy
(\Cref{fig1}), $\lambda_t$ converges to $\lambda^\ast$, and the cumulative
variance grows linearly (Figure~\ref{fig::eli5_variance}, Appendix~\ref{app:exp}).

\paragraph{Autoregressive regime.}
We also run \Cref{alg:eos} unchanged on genuine autoregressive generation,
which the analysis does not cover: 200 ELI5 questions, $n=1024$ tokens,
against a fixed-temperature grid and standard truncation samplers, compared at
matched realized cross-entropy (\Cref{app:exp}).  EOS stays close to the
fixed-temperature entropy frontier, losing $0.006$, $0.037$ and $0.088$ nats
   per token at $\gamma=0.5,0.9,1.3$, which reflects the small loss to the
   offline comparator that \Cref{thm:upperbound} allows.  Yet it is markedly less repetitive at the
same cross-entropy (\Cref{fig:dep}).  At essentially equal cross-entropy
($0.50$), EOS with $\gamma=0.5$ repeats $2.0\%$ of 16-word blocks (with at
most $10\%$ edits) against $3.4\%$ for nucleus sampling with $p=0.95$.
Against the fixed-temperature frontier at $\gamma=0.9$, seq-rep-4 falls from
$0.068$ to $0.043$, repeated 16-word blocks from $1.6\%$ to $0.9\%$,
near-duplicate sentences from $3.4\%$ to $2.1\%$, and cross-sample 4-gram
overlap from $0.019$ to $0.011$, with 95\% intervals excluding zero at all
   budgets and at every edit tolerance (Figure~\ref{fig:dep-extra},
   Appendix~\ref{app:exp}).  Under feedback, plain EOS overshoots its budget:
   its realized cross-entropy is $1.01$ at $\gamma=0.9$ and $1.53$ at
   $\gamma=1.3$.  The tightened variant corrects the mean, with a fifth of the
   per-prompt spread of a calibrated fixed temperature (Table~\ref{tab:budget},
   Appendix~\ref{app:exp}).
\begin{figure}[H]
  \centering
  \includegraphics[width=0.48\linewidth]{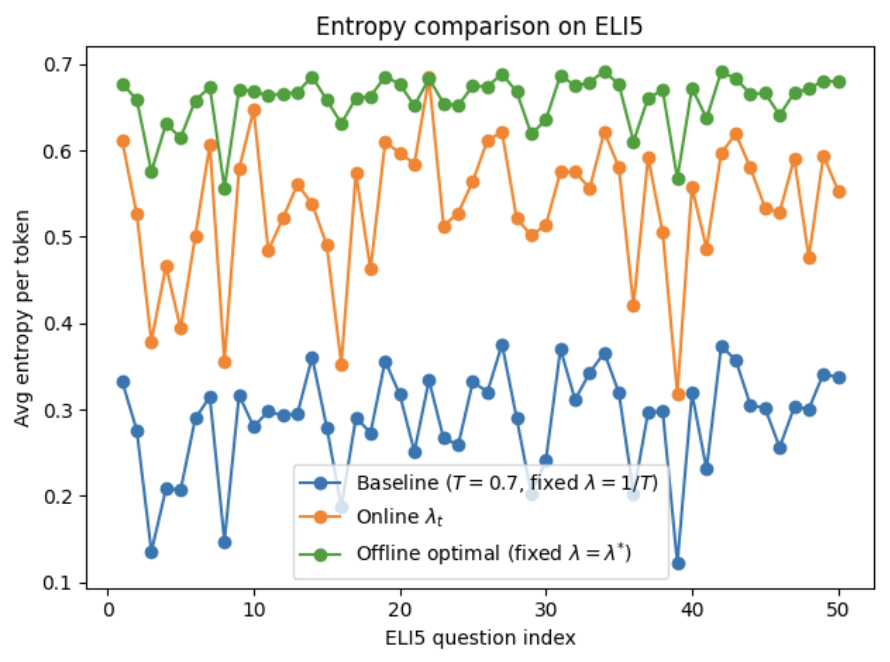}\hfill
  \includegraphics[width=0.48\linewidth]{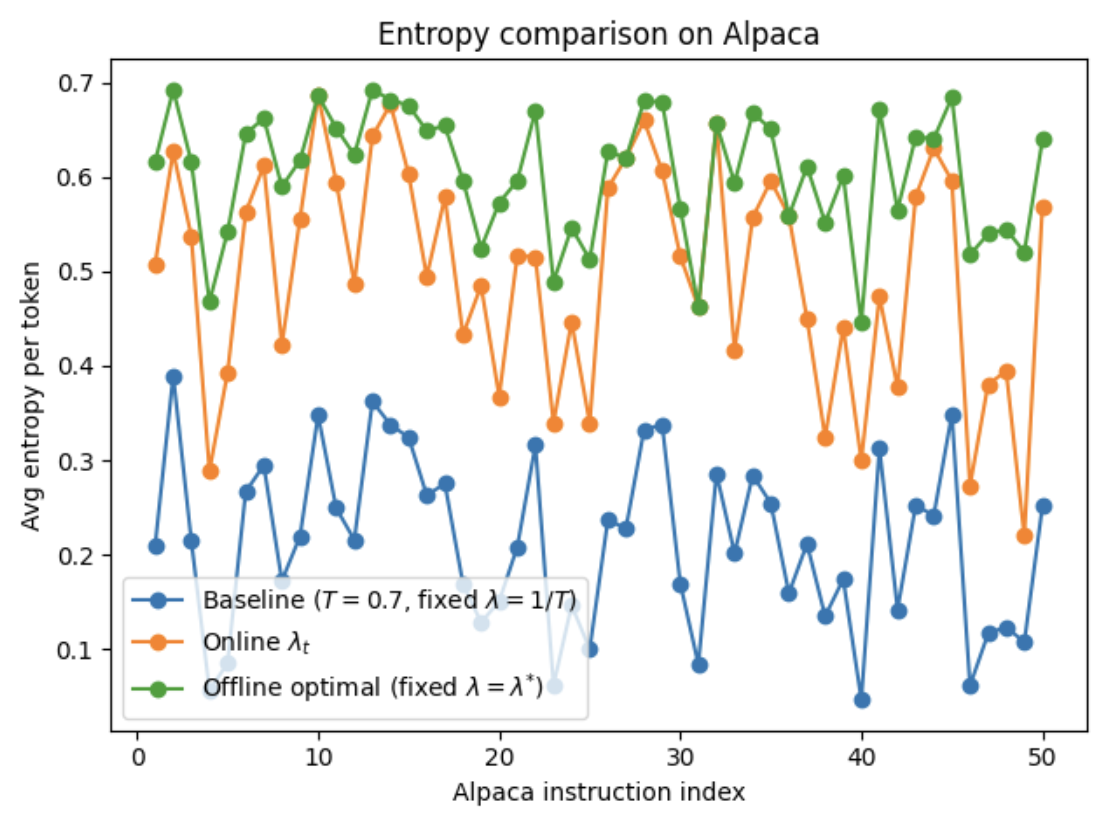}
  \caption{Exogenous regime: per-token entropy on ELI5 (left) and Alpaca
  (right).  The offline optimum (green) is highest; EOS (orange) tracks it with
  a small gap, as \Cref{thm:offline,thm:upperbound} predict; the fixed
  temperature $T=0.7$ (blue) is budget-feasible but leaves budget unspent.}
  \label{fig1}
\end{figure}

\begin{figure}[H]
  \centering
  \includegraphics[width=0.70\linewidth]{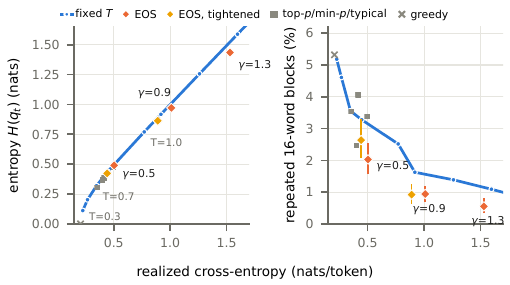}
  \caption{Autoregressive regime (Llama-3-8B-Instruct, ELI5).  Entropy
  (left) and repeated 16-word blocks with $\le10\%$ edits (right) against
  realized cross-entropy; the curve is fixed temperature
  ($T\in[0.3,1.15]$).  Bars are 95\% bootstrap intervals over prompts.}
  \label{fig:dep}
\end{figure}

\section{Concluding Remarks}

We formulated decoding as online entropy maximization under a single
cross-entropy budget coupled across the horizon.  EOS attains $O(R\log n)$
entropy loss with $O(R\log n)$ slack, tightening its budget makes it exactly
feasible (given a drift bound) at the same order, and no exactly feasible online algorithm can lose
$o(R)$; in the i.i.d.-mixture regime $R=\tilde O(\sqrt n)$.  Empirically, EOS
also reduces repetition under autoregressive feedback, which our
analysis does not cover.  Open directions are closing the $\log n$ gap, extending the
analysis to autoregressive feedback and its future-partition-function tilt
(\Cref{sec:setup}), and selecting $\gamma$ from a downstream
\emph{pass@$k$}~\citep{chen2021evaluating} target.
\bibliography{references_main,ml,token_extra}

\begin{thebibliography}{}
\bibitem[Meister and Cotterell, 2023]{meister2023beam}
Meister, C. and Cotterell, R. (2023).
\newblock If beam search is the answer, what was the question?
\newblock {\em Transactions of the Association for Computational Linguistics},
  11:152--171.
\end{thebibliography}

\clearpage
\appendix
\thispagestyle{empty}
\onecolumn

\section{Basic Facts}
\label{app:facts}

\begin{proposition}
\label{prop:derivative}
Let \Cref{as:rho} hold and let $f_t,g_t$ be as in \eqref{eq:ftgt}.  Then, with
$L=\log(1/\rho)$,
\[
g_t'(\lambda)=\Var_{q^{\lambda}_t}\!\left[\log\tfrac{1}{p_t(w)}\right]\in[0,L^2],
\qquad
f_t'(\lambda)=-\lambda\,g_t'(\lambda)\in[-\Lambda L^2,\,0],
\qquad
|g_t''(\lambda)|\ \le\ L^{3}.
\]
\end{proposition}

\begin{proof}
Since $q^\lambda_t(w)\propto e^{\lambda\log p_t(w)}$ is a one-parameter exponential
family with natural parameter $\lambda$ and sufficient statistic $\log p_t(w)$, the
cumulant generating function $\psi_t(\lambda)=\log\sum_w p_t^\lambda(w)$ satisfies
$\psi_t'=\E_{q^\lambda_t}[\log p_t]$ and $\psi_t''=\Var_{q^\lambda_t}[\log p_t]$.  As
$g_t(\lambda)=\gamma+\psi_t'(\lambda)$, we get $g_t'=\psi_t''=\Var[\log(1/p_t)]\ge0$;
the range bound follows since $\log(1/p_t)\in[0,L]$ under \Cref{as:rho}, so the
variance is at most $L^2/4\le L^2$.  Also
$g_t''=\psi_t'''=\E[(\log p_t-\E\log p_t)^3]$, bounded in absolute value by $L^3$.
For $f_t$, $H(q^\lambda_t)=\psi_t(\lambda)-\lambda\psi_t'(\lambda)$, whence
$f_t'(\lambda)=-\lambda\psi_t''(\lambda)=-\lambda g_t'(\lambda)$, which lies in
$[-\Lambda L^2,0]$ for $\lambda\in[0,\Lambda]$.
\end{proof}

\begin{proof}[Proof of \Cref{lem:lambdabound}]
Fix $t$ and put $A_t(\lambda)=\sum_{i\le t}g_i(\lambda)$.  Nonnegativity of
$\lambda_t$ and $\lambda^\ast$ is the right-hand inequality of \Cref{as:feasible}.
Under the additional hypothesis, \Cref{prop:derivative} gives $A_t'\ge ct$ on
$[0,\Lambda]$.  At $\lambda=0$,
$g_i(0)=\gamma-\frac{1}{|\W|}\sum_w\log\frac{1}{p_i(w)}\ge \gamma-L$, so
$A_t(\Lambda)\ge A_t(0)+ct\Lambda\ge -t(L-\gamma)+t(L-\gamma)=0$; since $A_t$ is
nondecreasing, its root lies in $[0,\Lambda]$.  The same argument with $t=n$ bounds
$\lambda^\ast$.
\end{proof}

\begin{proof}[Proof of \Cref{lem:envelope}]
$\mathrm{OPT}(b)$ is the value of a concave maximization over a convex set whose
constraint is relaxed linearly in $b$; concavity and monotonicity are immediate, and
the derivative equals the optimal multiplier by standard Lagrangian sensitivity
(the multiplier is unique under \Cref{as:curvature}, so $\mathrm{OPT}$ is
differentiable).  The displayed inequality is concavity plus $\mathrm{OPT}'\ge0$.
\end{proof}

\section{Proof of Theorem~\ref{thm:probabilityexplanation}}
\label{app:prob}

\begin{proof}
Let $X_t=\log p_t(w_t)+\gamma$ and $D_t=X_t-\E_{q_t}[X_t\mid w^{t-1}]$, so that
$S_t=\sum_{i\le t}D_i$ is a martingale.  Under \Cref{as:rho}, $\log p_t(w_t)$ takes
values in $[-L,0]$, so $D_t$ lies in an interval of length at most $L$.  By
Azuma--Hoeffding for martingales with bounded ranges,
\[
\Pr\big(|S_n|\ge\epsilon\big)\ \le\ 2\exp\!\left(-\frac{2\epsilon^2}{nL^2}\right),
\]
so with probability at least $1-\delta$, $|S_n|\le L\sqrt{(n/2)\log(2/\delta)}$.
Since
\[
\log P(w^n)+\gamma n=\sum_{t\in[n]}X_t
=\sum_{t\in[n]}\sum_{w}q_t(w)\log p_t(w)+\gamma n+S_n ,
\]
part (i) follows.  For (ii), take expectations: on the event of probability
$1-\delta$ we have $\log P(w^n)\ge -\gamma n-\Delta_n$, and on its complement
$\log P(w^n)\ge -nL$ deterministically, so
$\E[\log P(w^n)]\ge -\gamma n-\Delta_n-\delta nL$.  Since $\E[S_n]=0$,
$\E[\log P(w^n)]=\E\big[\sum_t\sum_w q_t(w)\log p_t(w)\big]$.  Since the $p_t$ are
exogenous, the $q_t$ do not depend on $w^{t-1}$, so the right side is the left side of
(ii), which gives the claim (the $L\sqrt{\cdot}$ term is not even needed).
\end{proof}

\section{Proof of Theorem~\ref{thm:offline}}
\label{app:offline}

\begin{proof}
The objective $\sum_t H(q_t)$ is concave and the constraint is linear, so
\eqref{eq:offline} is convex; \Cref{as:feasible} gives a Slater point, so strong
duality and the KKT conditions hold.  With multipliers $\lambda\ge0$ for the budget
and $\mu_t$ for normalization, the Lagrangian is
\[
\mathcal{L}=\sum_{t}\mu_t\Big(1-\sum_w q_t(w)\Big)-\sum_t H(q_t)
-\lambda\Big(\sum_t\sum_w q_t(w)\log p_t(w)+\gamma n\Big),
\]
and $\partial\mathcal{L}/\partial q_t(w)=\log q_t(w)+1-\lambda\log p_t(w)-\mu_t$.
Setting this to zero gives $q_t(w)\propto p_t^{\lambda}(w)$, i.e.\
$q_t=q^\lambda_t$.  Note that $\lambda$ does not depend on $t$: there is a single
constraint, hence a single multiplier, hence a single temperature.  Complementary
slackness and \Cref{as:feasible} force the constraint to bind, which is
\eqref{eq:lambdastar}; uniqueness follows from strict monotonicity of
$\lambda\mapsto\sum_t g_t(\lambda)$ under \Cref{as:curvature}
(\Cref{prop:derivative}).
\end{proof}

\section{Proof of Theorem~\ref{thm:upperbound}}
\label{app:upperbound}

\begin{proof}
Write $A_t(\lambda)=\sum_{i\le t}g_i(\lambda)$, let $\tilde\lambda_t$ be a root of
$A_t$, so that $\lambda_t=\Pi_{[0,\Lambda]}(\tilde\lambda_t)$ by \eqref{eq:lambdat} and
the clipping in \Cref{alg:eos}, and note $|A_t(\lambda^\ast)|\le R$ by \Cref{def:R}.

\textbf{Step 1: $\lambda_t$ tracks $\lambda^\ast$.}
By Taylor's theorem with Lagrange remainder,
\begin{equation}
\label{eq:taylor}
A_t(\lambda)=A_t(\lambda^\ast)+\Big(\sum_{i\le t}g_i'(\lambda^\ast)\Big)(\lambda-\lambda^\ast)+L_t(\lambda),
\qquad
|L_t(\lambda)|\le \tfrac{1}{2}\,t\,L^{3}\,(\lambda-\lambda^\ast)^2,
\end{equation}
using $|g_i''|\le L^3$ from \Cref{prop:derivative}.
Take
$\lambda=\lambda^\ast+\tfrac{2R}{ct}$.  By \Cref{as:curvature},
$\sum_{i\le t}g_i'(\lambda^\ast)\ge ct$, hence
\[
A_t\Big(\lambda^\ast+\tfrac{2R}{ct}\Big)\ \ge\ -R+2R-\frac{2R^2L^3}{c^2t}\ >\ 0
\qquad\text{whenever } t\ \ge\ T_0:=\Big\lceil \tfrac{8RL^{3}}{c^{2}}\Big\rceil .
\]
Symmetrically $A_t(\lambda^\ast-\tfrac{2R}{ct})<0$ for $t\ge T_0$.  Since $A_t$ is
nondecreasing, a root $\tilde\lambda_t$ lies in $[\lambda^\ast-\tfrac{2R}{ct},\lambda^\ast+\tfrac{2R}{ct}]$,
and since $\lambda^\ast\in[0,\Lambda]$ the projection onto $[0,\Lambda]$ does not
increase the distance to $\lambda^\ast$, so
\begin{equation}
\label{eq:track}
\big|\lambda_t-\lambda^\ast\big|\ \le\ \frac{2R}{ct}\qquad\text{for all } t\ge T_0 .
\end{equation}

\textbf{Step 2: entropy loss.}  Split the horizon at $T_0$.  For $t<T_0$ bound each
term crudely by the maximum possible entropy difference, $\log|\W|$; there are
$T_0=O(R)$ such terms, contributing $O(R\log|\W|)$.  For $t\ge T_0$, by
\Cref{prop:derivative} and a second-order expansion of $f_t$,
\[
f_t(\lambda^\ast)-f_t(\lambda_t)
\ \le\ |f_t'(\lambda^\ast)|\,|\lambda_t-\lambda^\ast|+\tfrac12\max_\lambda|f_t''(\lambda)|\,|\lambda_t-\lambda^\ast|^2,
\]
with $|f_t'|\le \Lambda L^2$ and $|f_t''|=|g_t'+\lambda g_t''|\le L^2+\Lambda L^3$.
Using \eqref{eq:track},
\[
\sum_{t\ge T_0}\Lambda L^{2}\frac{2R}{ct}\;=\;O(R\log n),
\qquad
\sum_{t\ge T_0}\tfrac12(L^{2}+\Lambda L^{3})\frac{4R^{2}}{c^{2}t^{2}}
\;\le\;O\!\Big(\frac{R^{2}}{T_0}\Big)\;=\;O(R),
\]
where the second sum uses $\sum_{t\ge T_0}t^{-2}\le 1/(T_0-1)$ and $T_0=\Theta(R)$.
Adding the three contributions gives \eqref{eq:divupper}.

\textbf{Step 3: global slack.}  The slack is $-\sum_{t}g_t(\lambda_t)$.  For
$t<T_0$, $|g_t|\le L$, contributing $O(R)$.  For $t\ge T_0$, expand $g_t$ around
$\lambda^\ast$ and use $\sum_t g_t(\lambda^\ast)$ bounded by $R$ in absolute value
together with $|g_t'|\le L^2$, $|g_t''|\le L^3$ and \eqref{eq:track}:
\[
\Big|\sum_{t\ge T_0}g_t(\lambda_t)\Big|
\ \le\ R+\sum_{t\ge T_0}L^{2}\frac{2R}{ct}+\sum_{t\ge T_0}\tfrac12 L^{3}\frac{4R^{2}}{c^{2}t^{2}}
\ =\ O(R\log n),
\]
which is \eqref{eq:violationupper}.
\end{proof}

\section{Proof of Theorem~\ref{thm:exactfeasible}}
\label{app:exactfeasible}

\begin{proof}
(i) \Cref{thm:upperbound} applied at budget $\gamma_\varepsilon$ gives
\[
-\sum_t\sum_w q_t(w)\log p_t(w)\ \le\ \gamma_\varepsilon n + C R_\varepsilon\log n
\ =\ \gamma n-\varepsilon n + C R_\varepsilon\log n\ \le\ \gamma n
\]
whenever $\varepsilon n\ge C R_\varepsilon\log n$.

(ii) By \Cref{thm:upperbound} at budget $\gamma_\varepsilon$, the entropy of the
output is at least $\mathrm{OPT}(\gamma_\varepsilon n)-CR_\varepsilon\log n$.  By
\Cref{lem:envelope} with $b=\gamma n$ and $b'=\gamma_\varepsilon n$,
\[
\mathrm{OPT}(\gamma n)-\mathrm{OPT}(\gamma_\varepsilon n)\ \le\ \lambda^\ast(\gamma_\varepsilon n)\,\varepsilon n\ \le\ \Lambda_\varepsilon\varepsilon n
\]
using \Cref{lem:lambdabound}.  Adding the two displays gives the claim.  Setting
$\varepsilon=CR_\varepsilon\log n/n$ satisfies the hypothesis of (i) with equality
and makes the second term $\Lambda_\varepsilon C R_\varepsilon\log n$.
\end{proof}

\section{Proof of the Lower Bound, Theorem~\ref{thm:lowerbound}}
\label{app:lowerbound}

W.l.o.g.\ let $\W=[|\W|]$ and let $M$ be the smallest integer with
$-\frac{(M-1)\log\frac{1-e^{-\gamma/3}}{M-1}}{M}+\frac{\gamma}{3M}>\frac{4\gamma}{3}$
and $\log M\ge \frac{4\gamma}{3}$.  For $\tfrac1M\le x\le 1$ set
\[
d_x(w)=\begin{cases} x,& w=M\\ \frac{1-x}{M-1},& w\in[M-1],\end{cases}
\qquad
C(\lambda,x)=-\sum_{w}\frac{d^\lambda_x(w)}{\sum_{w'}d^\lambda_x(w')}\log d_x(w).
\]
Then $C(0,e^{-\gamma/3})>\tfrac{4\gamma}{3}$ and
$C(\infty,e^{-\gamma/3})\le\tfrac{\gamma}{3}<\tfrac{2\gamma}{3}$, so by continuity
there are $\lambda_1,\lambda_2\ge0$ with
\begin{equation}
\label{eq:d1}
C(\lambda_1,e^{-\gamma/3})=\tfrac{4\gamma}{3},
\qquad
C(\lambda_2,e^{-\gamma/3})=\tfrac{2\gamma}{3}.
\end{equation}
Since $C(\lambda,\tfrac1M)=\log M\ge\tfrac{4\gamma}{3}$ and $C(\lambda,1)=0$, there
are $x_1,x_2\in[\tfrac1M,1]$ with
\begin{equation}
\label{eq:d2d3}
C(\lambda_1,x_1)=\tfrac{2\gamma}{3},\qquad C(\lambda_2,x_2)=\tfrac{4\gamma}{3},
\end{equation}
and $x_3,x_4$ with $C(\lambda_1,x_3)=C(\lambda_2,x_4)=\gamma$.  Define
\begin{equation}
\label{eq:pt1pt2}
p_{t,1}=\begin{cases} d_{e^{-\gamma/3}}& t\le \tfrac{3R}{\gamma}\\
d_{x_1}& \tfrac{3R}{\gamma}<t\le\tfrac{6R}{\gamma}\\ d_{x_3}&\text{else}\end{cases}
\qquad
p_{t,2}=\begin{cases} d_{e^{-\gamma/3}}& t\le \tfrac{3R}{\gamma}\\
d_{x_2}& \tfrac{3R}{\gamma}<t\le\tfrac{6R}{\gamma}\\ d_{x_4}&\text{else.}\end{cases}
\end{equation}
A direct computation from \eqref{eq:d1}--\eqref{eq:pt1pt2} gives
$\lambda^\ast_\ell=\lambda_\ell$ for $\ell\in\{1,2\}$ and
$$\max_{t}|\sum_{i\le t}\sum_w Q^{\off}_i(w)\log p_i(w)+\gamma t|=R$$ on both
instances, so both have drift parameter exactly $R$.  Taking $\W=[M]$ (the remaining
symbols removed by the pre-filtering behind \Cref{as:rho}), both instances satisfy
\Cref{as:rho,as:feasible,as:curvature} with $\rho,c$ depending only on $(\gamma,M)$,
not on $R$ or $n$: every $p_{t,\ell}$ is one of five fixed full-support distributions
$d_x$, and each round's variance under $q^{\lambda_\ell}_t$ is positive.

The two instances agree on $t\le 3R/\gamma=:\alpha$, so any online $\A$ produces the
same $q^\A_t$ there.  Writing $B:=\sum_{t\le\alpha}\sum_w q^\A_t(w)\log p_{t,1}(w)+2R$,
so $R+B$ is the prefix value of $\A$ in the constraint of $S$ below, while the offline
prefix value is $-R$ on instance $1$ and $+R$ on instance $2$.  We treat $B\ge R/2$ on
instance $2$; the cases $B\le -R/2$ and $|B|<R/2$ are handled at the end.

Let $S(\{p_i\}_{i\le t},x)$ be the value of
$\max\sum_{i\le t}H(q_i)$ subject to $\sum_{i\le t}\sum_w q_i(w)\log p_i(w)+\gamma t\ge x$,
attained by $q^{\lambda^\ast_x}_i$ where $\lambda^\ast_x$ makes the constraint bind.
Since $\A\in\A_0$ is globally feasible, its tail value is at least $-R-B$, and $S$ is
nonincreasing in $x$, so
\[
\sum_{t=1}^n H(q^\A_t)\ \le\ S(\{p_i\}_{i\le\alpha},R+B)+S(\{p_i\}_{i>\alpha},-R-B).
\]
With $\lambda^\ast_R=\lambda^\ast_{-R}=\lambda_2$ (the offline optimum of instance $2$
has a single multiplier, and its prefix and tail values are $R$ and $-R$) and $f_t'=-\lambda g_t'$, the same chain of identities as in the
main text yields
\begin{equation}
\label{eq:lowerboundmid}
\sum_{t=1}^n H(Q^{\off}_t)-\sum_{t=1}^n H(q^{\A}_t)
\ \ge\ \sum_{i\le \alpha}\int_{\lambda^\ast_R}^{\lambda^\ast_{R+B}}(\lambda-\lambda^\ast_R)\,g_i'(\lambda)\,d\lambda .
\end{equation}

\paragraph{A variance lower bound.}  On the common prefix, $p_i=d_{e^{-\gamma/3}}$,
so $\log(1/p_i(w))$ takes only the two values
$a=\gamma/3$ (at $w=M$) and $b=\log\frac{M-1}{1-e^{-\gamma/3}}\ge\log(M-1)\ge
\tfrac{4\gamma}{3}$.  Let $P=q^\lambda_i(M)$ and
$\mu=\E_{q^\lambda_i}[\log(1/p_i)]=Pa+(1-P)b$.  For
$\lambda^\ast_R\le\lambda\le\lambda^\ast_{3R/2}$ one has
$\mu\in[\tfrac{\gamma}{2},\tfrac{2\gamma}{3}]$ (this is the definition of
$\lambda^\ast_{3R/2}$ together with monotonicity of $g$).  Then
\[
1-P=\frac{\mu-a}{b-a}\ \ge\ \frac{\gamma/6}{b-a},
\qquad
P=\frac{b-\mu}{b-a}\ \ge\ \frac{2\gamma/3}{b-a},
\]
so that
\begin{equation}
\label{eq:varlb}
g_i'(\lambda)=\Var_{q^\lambda_i}\!\left[\log\tfrac{1}{p_i}\right]=P(1-P)(b-a)^2
\ \ge\ \frac{2\gamma}{3}\cdot\frac{\gamma}{6}\ =\ \frac{\gamma^{2}}{9}.
\end{equation}

\paragraph{The gap in $\lambda$.}  Let $\Delta:=\lambda^\ast_{3R/2}-\lambda^\ast_R$.
By definition of these roots,
\[
\int_{\lambda^\ast_R}^{\lambda^\ast_{3R/2}}\sum_{i\le\alpha}g_i'(\lambda)\,d\lambda\;=\;\frac{R}{2},
\]
and $\sum_{i\le\alpha}g_i'(\lambda)\le \alpha G$ with
$G=(\tfrac{4\gamma}{3}+\log\tfrac{1}{1-e^{-\gamma/3}})^2$ an upper bound on each
$g_i'$.  Since $\alpha=3R/\gamma$,
\begin{equation}
\label{eq:lowerlambdadiff}
\Delta\ \ge\ \frac{R/2}{(3R/\gamma)G}\ =\ \frac{\gamma}{6G}.
\end{equation}
Note that $\Delta=\Theta(1)$, not $\Theta(R)$: the prefix has $\Theta(R)$ rounds, each
contributing curvature, so a constraint shift of $R/2$ moves the multiplier by only a
constant.

\paragraph{Conclusion.}  Combining \eqref{eq:lowerboundmid}, \eqref{eq:varlb} and
\eqref{eq:lowerlambdadiff}, and using $\lambda^\ast_{R+B}\ge\lambda^\ast_{3R/2}$
(since $B\ge R/2$),
\[
\sum_{t=1}^n H(Q^{\off}_t)-\sum_{t=1}^n H(q^{\A}_t)
\ \ge\ \alpha\cdot\frac{\gamma^{2}}{9}\int_{0}^{\Delta}u\,du
\ =\ \frac{3R}{\gamma}\cdot\frac{\gamma^{2}}{9}\cdot\frac{\Delta^{2}}{2}
\ \ge\ \frac{R\gamma}{6}\cdot\frac{\gamma^{2}}{36G^{2}}
\ =\ \frac{\gamma^{3}}{216\,G^{2}}\;R ,
\]
which is \eqref{eq:divlower}.

\paragraph{Remaining cases.}  Write $y\mapsto\lambda(y)$ for the prefix root at value
$y$ (nondecreasing, $\lambda(-R)=\lambda_1$, $\lambda(R)=\lambda_2$); the substitution
$dy=\sum_{i\le\alpha}g_i'(\lambda)\,d\lambda$ turns \eqref{eq:lowerboundmid} into
$\int_{y^{\off}}^{y^{\A}}\big(\lambda(y)-\lambda(y^{\off})\big)\,dy$.
(a) $B\le-R/2$, instance $2$: on $[\lambda(R/2),\lambda(R)]$ one has
$\mu\in[\tfrac{2\gamma}{3},\tfrac{5\gamma}{6}]$, so the two-point computation gives
$g_i'\ge\tfrac{\gamma}{3}\cdot\tfrac{\gamma}{2}\ge\tfrac{\gamma^2}{9}$, and
$\lambda(R)-\lambda(R/2)\ge\gamma/(6G)$ as in \eqref{eq:lowerlambdadiff}; the
conclusion is identical.
(b) $|B|<R/2$, instance $1$: here $y^{\A}=R+B>R/2$ and $y^{\off}=-R$, so the loss is at
least $\int_0^{R/2}\big(\lambda(y)-\lambda(-R)\big)dy\ge\tfrac R2\big(\lambda(0)-\lambda(-R)\big)
\ge\tfrac R2\cdot\tfrac{R}{\alpha G}=\tfrac{\gamma R}{6G}\ge\kappa R$,
using $\sum_{i\le\alpha}g_i'\le\alpha G$ and $36G\ge\gamma^2$.
\qed

\section{Proof of Theorem~\ref{thm:pareto}}
\label{app:pareto}

\begin{proof}
Run the construction of \Cref{app:lowerbound} verbatim.  The only place exact
feasibility was used is the decomposition
$\sum_t H(q^\A_t)\le S(\{p_i\}_{i\le\alpha},R+B)+S(\{p_i\}_{i>\alpha},-R-B)$, which
encodes that a prefix overspend of $B$ must be repaid in the tail.  If $\A$ is
merely $V$-feasible, the tail budget is loosened by $V$ and the correct
decomposition is
\[
\sum_{t=1}^n H(q^\A_t)\ \le\ S(\{p_i\}_{i\le\alpha},R+B)+S(\{p_i\}_{i>\alpha},-R-B-V).
\]
The value function $S(\cdot,x)$ is nonincreasing in $x$ with
$\partial S/\partial x=-\lambda^\ast_x$ (\Cref{lem:envelope} applied to the tail
subproblem), and $\lambda^\ast_x\le\Lambda$ by the hypothesis of \Cref{thm:pareto}, so
\[
S(\{p_i\}_{i>\alpha},-R-B-V)-S(\{p_i\}_{i>\alpha},-R-B)\ \le\ \Lambda V .
\]
Subtracting this from the bound of \Cref{thm:lowerbound} gives \eqref{eq:pareto}.
\end{proof}

\section{Proof of Theorem~\ref{thm:o(n)}}

\begin{proof}
Let $g^{(i)}(\lambda)=\gamma+\sum_w q^{\lambda,(i)}(w)\log p^{(i)}(w)$ be the constraint
function of component $i$, let $t_i=|\{j\le t: p_j=p^{(i)}\}|$, and let
$\hat P_i=n_i/n$.  By \Cref{as:rho}, $|g^{(i)}(\lambda)|\le \log(1/\rho)=:L$ for
every $\lambda\ge0$.

\textbf{The key identity.}  By definition of $\lambda^\ast$
(\Cref{thm:offline}), $\sum_{t\le n}g_t(\lambda^\ast)=0$; grouping the rounds by
component,
\begin{equation}
\label{eq:keyidentity}
\sum_{i\in[k]}n\hat P_i\,g^{(i)}(\lambda^\ast)\;=\;0 .
\end{equation}

\textbf{Consequence.}  For every $t$,
\[
R_t=\Big|\sum_{i\in[k]}t_i\,g^{(i)}(\lambda^\ast)\Big|
=\Big|\sum_{i\in[k]}\big(t_i-t\hat P_i\big)g^{(i)}(\lambda^\ast)\Big|
\ \le\ L\sum_{i\in[k]}\big|t_i-t\hat P_i\big|,
\]
where the second equality uses \eqref{eq:keyidentity} scaled by $t/n$.  By the
triangle inequality,
\[
\sum_{i}\big|t_i-t\hat P_i\big|\ \le\ \sum_i\big|t_i-tP_i\big|+t\sum_i\big|\hat P_i-P_i\big| .
\]

\textbf{Uniform concentration.}  Apply \Cref{lem:empiricaltv} at each sample size
$t\in[n]$ with confidence $\delta/(n+1)$, and once more at $t=n$ for
$\sum_i|\hat P_i-P_i|$.  On the intersection event, which has probability at least
$1-\delta$ by a union bound over the $n+1$ applications, we have simultaneously for
all $t$
\[
\sum_i|t_i-tP_i|\ \le\ \sqrt{2t\Big(k\log2+\log\tfrac{n+1}{\delta}\Big)}
\quad\text{and}\quad
t\sum_i|\hat P_i-P_i|\ \le\ t\sqrt{\tfrac{2(k\log2+\log\frac{n+1}{\delta})}{n}} .
\]
Both are at most $\sqrt{2n(k\log2+\log\frac{n+1}{\delta})}$ for $t\le n$, so
\[
R=\max_{t\in[n]}R_t\ \le\ 2L\sqrt{2n\Big(k\log2+\log\tfrac{n+1}{\delta}\Big)} ,
\]
as claimed.
\end{proof}

\begin{remark}
The union bound over prefixes is what converts a statement about the empirical
frequencies at a single sample size into the bound on $R=\max_t R_t$ that the
theorem actually needs; it costs $\log(1/\delta)\to\log(n/\delta)$ and preserves the
$\tilde O(\sqrt n)$ rate.
\end{remark}

\section{Experimental Details}
\label{app:exp}

\paragraph{Curvature.}
\Cref{fig::eli5_variance} shows the cumulative variance under the offline
temperature in the exogenous ELI5 experiment.

\begin{figure}[h]
  \centering
  \includegraphics[width=0.6\linewidth]{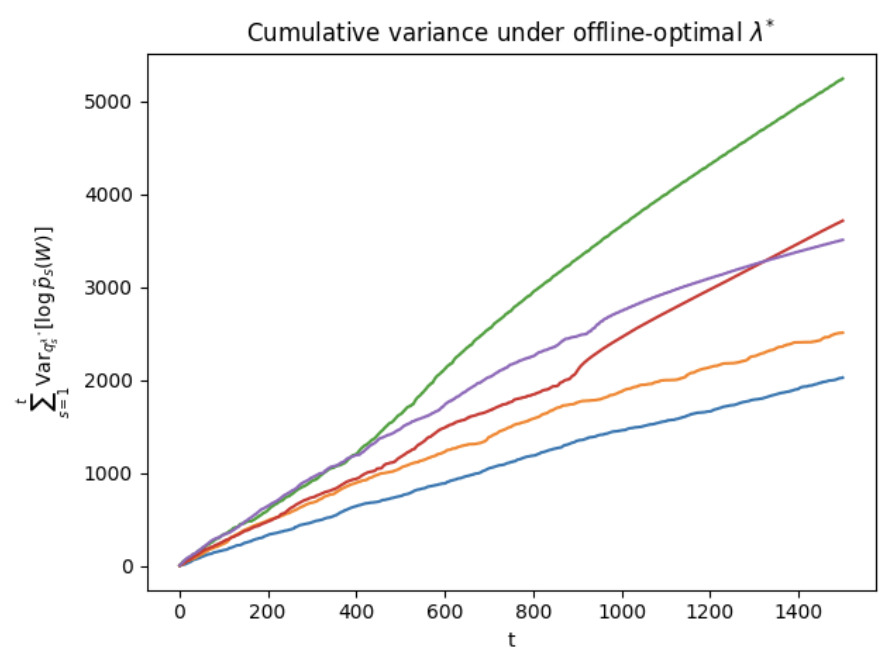}
\caption{Cumulative variance $\sum_{i\le t}\Var_{q_i^{\lambda^\ast}}[\log(1/p_i)]$
  against $t$.  Approximate linearity is direct evidence for
  \Cref{as:curvature}; $c$ may be read off as the slope.}
  \label{fig::eli5_variance}
\end{figure}

\paragraph{Autoregressive setup.}
Meta-Llama-3-8B-Instruct (bf16) generates $n=1024$ tokens for 200 ELI5
questions, 4 samples per prompt and method, on one AMD MI210 GPU per run, with end-of-sequence tokens
suppressed so every method is compared on the same horizon.  The support at
each step is $\W_t=\{w:p_t(w)\ge\rho\}$, $\rho=10^{-4}$, capped at the 256
most likely tokens, and $\Lambda=50$.  $\lambda_t$ solves \eqref{eq:lambdat}
to tolerance $10^{-6}t$ by Newton's method warm-started at $\lambda_{t-1}$,
using that $\lambda\mapsto\sum_{i\le t}g_i(\lambda)$ has exact derivative
$\sum_{i\le t}\Var_{q_i^\lambda}[\log(1/p_i)]$ (so the left-hand side of
\eqref{eq:lambdat} has derivative $-\sum_{i\le t}\Var_{q_i^\lambda}[\log(1/p_i)]$),
with a bisection safeguard (about two evaluations of the prefix sum per token).
Methods: EOS with $\gamma\in\{0.5,0.9,1.3\}$; its tightened variant
(\Cref{thm:exactfeasible}) with $\varepsilon=0.1\gamma$; a re-solving
heuristic that spreads the remaining budget
$\gamma n-\sum_{i<t}\mathrm{CE}_i$ evenly over the remaining rounds (not
covered by our theory); fixed temperature
$T\in\{0.3,0.5,0.7,0.85,1.0,1.03,1.06,1.09,1.12,1.15,1.3\}$; nucleus
$p\in\{0.9,0.95\}$~\citep{holtzman2020curious}; min-$p\in\{0.05,0.1\}$;
locally typical sampling ($0.9$); and greedy decoding.  For every method we
log $H(q_t)$ and $\mathrm{CE}_t=-\sum_w q_t(w)\log p_t(w)$ against the full
model distribution, from the raw logits, so the realized slack is on the same
scale for all methods.

\begin{figure}[H]
  \centering
  \includegraphics[width=0.8\linewidth]{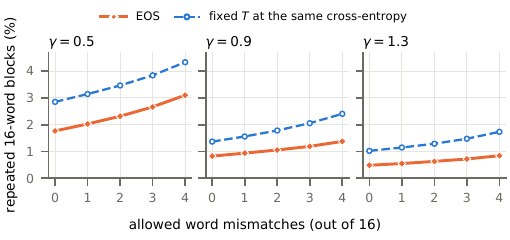}\\[6pt]
  \includegraphics[width=0.8\linewidth]{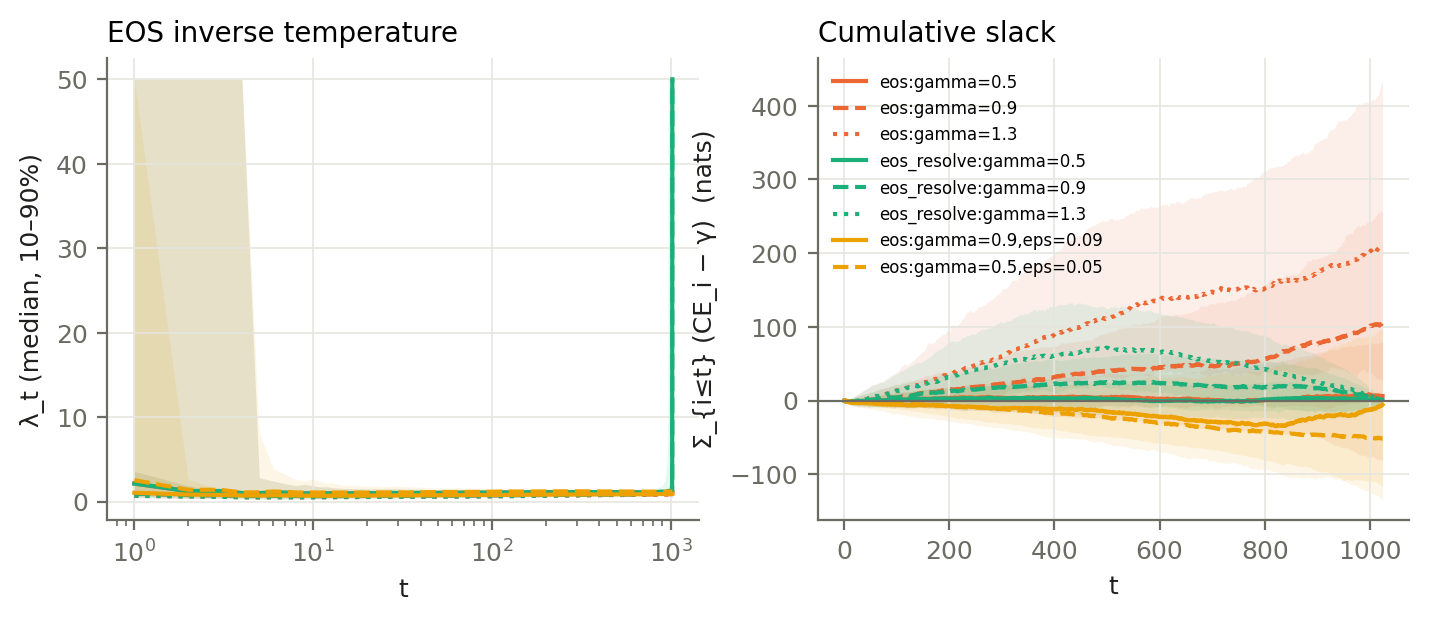}
  \caption{Top: repeated 16-word blocks as the number of allowed word
  mismatches grows, for EOS and for fixed temperature at the same realized
  cross-entropy; other samplers are compared at a single tolerance in
  \Cref{fig:dep}.  Bottom: EOS inverse temperature $\lambda_t$ and cumulative
  overshoot $\sum_{i\le t}(\mathrm{CE}_i-\gamma)$ (median, 10--90\% band).}
  \label{fig:dep-extra}
\end{figure}

\paragraph{Repetition metrics.}
Approximate repetition has no single definition, so we report a family of
metrics indexed by a tolerance, each computed on model tokens and on
normalized words (lower-cased, digits collapsed, punctuation removed).
(i) \emph{Exact recurrence}: $\Lambda_t$, the length of the longest string
starting at $t$ that already occurs earlier~\citep{wynerziv1989,ornsteinweiss1993};
\Cref{prop:recurrence} predicts $\Lambda_t\approx\log t/H$, and we report the
entropy-rate estimate $\hat H=[\mathrm{mean}_t\,(\Lambda_t+1)/\log t]^{-1}$,
   averaged over positions $2\le t\le0.95n$ whose match is not truncated by the end of the output.
(ii) \emph{Approximate recurrence}: for $m\in\{8,16,32\}$ and
$\tau\in[0,0.3]$, the fraction of positions whose length-$m$ block is within
Hamming distance $\tau m$ of an earlier block; $\tau=0$ is exact repetition.
(iii) \emph{Loops and $n$-grams}: fraction of the output in tandem repeats
(period $\le64$, at least three periods), seq-rep-$n$ and
rep-$\ell$~\citep{welleck2020neural}, and LZ78 and zstd compression ratios.
(iv) \emph{Sentence level}: fraction of sentences whose normalized edit
similarity to an earlier sentence exceeds $\theta=0.8$.
Diversity across the samples of a prompt (pairwise 4-gram Jaccard overlap) is
reported separately.    The headline approximate-repetition metric uses $m=16$ and $\tau=0.1$, i.e.\
   repeated 16-word blocks with at most one mismatched word.

\paragraph{Comparison protocol.}
Entropy and repetition both fall as the budget tightens, so all comparisons
are at matched realized cross-entropy: (a) the \emph{frontier}, every method
in the plane of realized cross-entropy against a metric, with the fixed-$T$
curve; (b) the \emph{offline comparator}, the fixed-$T$ curve interpolated at
each prompt's own realized cross-entropy, whose gap to EOS estimates the
entropy loss of \Cref{thm:upperbound}; (c) \emph{budget adherence}, a single
temperature calibrated to meet $\gamma$ on average over a held-out 25\% of
prompts, evaluated on the rest.  Intervals are 95\% bootstrap intervals over
prompts (2000 resamples).

\paragraph{Additional results.}
At $\gamma=0.9$ the match-length entropy-rate estimate rises from $2.63$
(fixed temperature at the same cross-entropy) to $2.83$ nats per word, and
LZ78 entropy from $3.71$ to $3.78$, so the reduction in repetition is also
visible in compression-based measures.  The reduction holds at every edit
tolerance, from exact repeats to blocks differing in a quarter of their words
(\Cref{fig:dep-extra}, top).  Nucleus and locally typical sampling are no
less repetitive than fixed temperature at matched cross-entropy, and min-$p$
($0.05$) only slightly.  The default $T=0.7$ spends $0.35$ nats per token
(entropy $0.32$), leaving most of the $\gamma=0.9$ budget unspent.  The
overshoot of plain EOS grows with the budget (realized cross-entropy $0.50$,
$1.01$, $1.53$ at $\gamma=0.5,0.9,1.3$); the tightened variant and the
re-solving heuristic correct it (\Cref{tab:budget}).  A single temperature
calibrated to $\gamma$ has a larger mean excess than plain EOS, because its
per-prompt cross-entropy varies almost five times as much.

\begin{table}[h]
  \centering\small
  \caption{Per-prompt budget adherence on 150 held-out prompts, $\gamma=0.9$.
  CE is realized cross-entropy (nats/token); sd is across prompts; excess is
  $\frac{1}{|\mathcal P|}\sum_{p}[\mathrm{CE}_p-\gamma]_+$ (prompts within
  budget count as zero).}
  \label{tab:budget}
  \begin{tabular}{lccccc}
    & mean CE & sd & \% over & excess & $H$ \\\hline
    fixed $T{=}1.03$ (calibrated) & 0.982 & 0.332 & 49 & 0.166 & 0.981 \\
    EOS                           & 1.012 & 0.071 & 95 & 0.112 & 0.973 \\
    EOS, tightened                & 0.890 & 0.068 & 39 & 0.021 & 0.863 \\
    re-solving (heuristic)        & 0.900 & 0.010 & 53 & 0.001 & 0.859
  \end{tabular}
\end{table}

\end{document}